\documentclass[11pt]{article}

\usepackage[T1]{fontenc}
\usepackage{amsmath,amssymb,amsthm,mathtools,microtype}
\usepackage[margin=0.98in]{geometry}
\usepackage[dvipsnames]{xcolor}
\usepackage[nocompress]{cite}
\usepackage{algorithm,algorithmic}
\usepackage[hidelinks]{hyperref}
\usepackage[capitalise,noabbrev,nameinlink]{cleveref}
\usepackage[most]{tcolorbox}
\tcbuselibrary{listings,breakable}
\hypersetup{
  pdftitle={An Optimal Agnostic PAC Algorithm},
  pdfauthor={Markus Engelund Mathiasen, Jian Qian, Nikita Zhivotovskiy}
}

\newtcblisting{promptbox}{
  breakable,
  listing only,
  colback=gray!3,
  colframe=black!60,
  boxrule=0.5pt,
  arc=1mm,
  left=6pt,
  right=6pt,
  top=6pt,
  bottom=6pt,
  title={Full prompt for \Cref{thm:pac}, GPT-5.6 Sol Pro, August 6, 2026},
  fonttitle=\bfseries\small,
  listing options={
    basicstyle=\ttfamily\fontsize{6pt}{6.5pt}\selectfont,
    breaklines=true,
    columns=fullflexible,
    keepspaces=true,
    showstringspaces=false,
    upquote=true
  }
}

\newtheorem{theorem}{Theorem}
\newtheorem{lemma}{Lemma}[section]
\theoremstyle{remark}
\newtheorem{remark}[lemma]{Remark}
\crefname{theorem}{theorem}{theorems}
\Crefname{theorem}{Theorem}{Theorems}
\crefname{lemma}{lemma}{lemmas}
\Crefname{lemma}{Lemma}{Lemmas}
\crefname{remark}{remark}{remarks}
\Crefname{remark}{Remark}{Remarks}
\crefname{algorithm}{algorithm}{algorithms}
\Crefname{algorithm}{Algorithm}{Algorithms}

\newcommand{\E}{\mathbb E}
\newcommand{\Prob}{\mathbb P}
\newcommand{\F}{\mathcal F}
\newcommand{\Hc}{\mathcal H}
\newcommand{\one}{\mathbf 1}
\newcommand{\Rad}{\mathfrak R}
\newcommand{\VC}{\operatorname{VC}}
\newcommand{\DIS}{\operatorname{DIS}}
\newcommand{\rhoF}{\rho_{\F}}

\title{An Optimal Agnostic PAC Algorithm}
\author{%
Markus Engelund Mathiasen\thanks{Department of Computer Science, Aarhus University.
  \texttt{markusm@cs.au.dk}.}
\and
Jian Qian\thanks{Division of Artificial Intelligence and Data Science,
  School of Computing and Data Science, The University of Hong Kong.
  \texttt{jianqian@hku.hk}.}
\and
Nikita Zhivotovskiy\thanks{Department of Statistics, University of
  California, Berkeley. \texttt{zhivotovskiy@berkeley.edu}.}}
\date{August 6, 2026}

\begin{document}
\maketitle

\begin{abstract}
Let $\Hc\subseteq\{-1,+1\}^{\mathcal X}$ be a class of finite VC dimension
$d\ge1$.
Writing $L$ for the binary risk and $L^*=\min_{h\in\Hc}L(h)$, we construct
a learner achieving the statistically optimal risk bound: from an i.i.d.\
sample of size $n$, for every
$0<\delta\le 1/2$, with probability at least $1-\delta$,
\[
 L(\widehat h)
 \le L^*+
7\cdot10^8\left(
 \sqrt{\frac{L^*(d+\log(1/\delta))}{n}}
 +\frac{d+\log(1/\delta)}{n}
 \right).
\]
This settles the sample complexity of agnostic PAC learning up to universal
constants at every fixed $L^*$,
matching the lower bounds of
Devroye, Gy\"orfi, and Lugosi
[\emph{A Probabilistic Theory of Pattern Recognition}, Springer, 1996].
\end{abstract}

\section{Introduction}

Determining sharp risk bounds in PAC classification has been one of the
central problems in statistical learning theory since the introduction of VC
theory and the PAC model
\cite{VapnikChervonenkis1971,Valiant1984,
BlumerEhrenfeuchtHausslerWarmuth1989}.
Throughout, let $\Hc\subseteq\{-1,+1\}^{\mathcal X}$ be a class of measurable
hypotheses with $\VC(\Hc)=d\ge1$.  Let
$S=((X_i,Y_i))_{i=1}^n$ be an i.i.d.\ sample from the distribution of
$(X,Y)$.  We write $L(h)=\Prob(h(X)\ne Y)$ and
$L^*=\min_{h\in\Hc}L(h)$.\footnote{We assume that all maps depending on the data are
measurable and that this minimum is attained.}

Historically, bounds in this setting have been proved by two rather different
methods.  In the agnostic setting, when $L^*$ is bounded away from zero by a
constant,
empirical risk minimization is analyzed through classical empirical process
methods, including chaining and covering bounds for VC classes.  This
approach gives the optimal $\sqrt{d/n}$ rate in this regime.

In the realizable, or noiseless, case, the uniform convergence
analysis of empirical risk minimization does not give the optimal $d/n$ rate.
The seminal work of Haussler, Littlestone, and Warmuth is based on
the one-inclusion graph, and its leave-one-out analysis gives the optimal
expected rate \cite{HLW1994}.  The first optimal high probability PAC bound
was proved by Hanneke \cite{Hanneke2016Optimal}, building on the work of Simon
\cite{Simon2015}.  Several later constructions also attain
the optimal PAC rate, including bagging \cite{Larsen2023Bagging}, aggregation
of one-inclusion rules \cite{AdenAliEtAl2023Optimal}, and, very recently, a majority of three
consistent classifiers \cite{RawalZhivotovskiy2026}.  These results rely on
specific combinatorial or probabilistic arguments rather than a direct
empirical process analysis.

The remaining question is how the optimal rate changes between the
realizable and fully agnostic settings, and which algorithms attain the
intermediate rates.  The relevant expectation and deviation lower
bounds go back to Devroye, Gy\"orfi, and Lugosi
\cite{DevroyeLugosi1995,DevroyeGyorfiLugosi1996}.\footnote{The exact form of
\eqref{eq:minimax-lower} is not stated in these references.  Within the usual
nontrivial boundaries for $L^*$, it follows by retaining the bounded witness
in their dimension dependent construction and combining it with their
deviation and realizable lower bounds.  Audibert later sharpened the
corresponding expected risk bound \cite{Audibert2009}.}
There is a universal constant $c>0$ such that, for every learner, including a
randomized or improper learner, and every admissible choice of $d,n,\delta$,
and $L^*$ within the usual nontrivial boundaries, there are an instance space,
a class $\Hc$ of VC dimension $d$, and a distribution satisfying
$\min_{h\in\Hc}L(h)=L^*$ for which, with probability at least $\delta$,
\begin{equation}\label{eq:minimax-lower}
 L(\widehat h)\ge L^*+c\left(
 \sqrt{\frac{L^*(d+\log(1/\delta))}{n}}
 +\frac{d+\log(1/\delta)}{n}\right).
\end{equation}
The probability is over the sample and any internal randomness of the
learner.  This bound gives the optimal realizable rate when $L^*=0$, the usual
agnostic rate when $L^*$ is bounded away from zero, and the intermediate
rates between them.  It remained open whether any learner could attain it.

Several upper bounds nearly match this lower bound.  Hanneke, Larsen, and Zhivotovskiy
\cite{HannekeLarsenZhivotovskiy2024} proved
\[
 L(\widehat h)
 \le L^*+C\left(
 \sqrt{\frac{L^*(d+\log(1/\delta))}{n}}
 +\frac{\log^5(n/d)(d+\log(1/\delta))}{n}
 \right),
\]
where $C>0$ is a universal constant.  This is suboptimal by the
polylogarithmic factor in the fast rate term.  By comparison, the
ERM bound incurs a multiplicative $\sqrt{\log(1/L^*)}$ factor in the VC part of the square
root term and a $\log(n/d)$ factor in the VC part of the fast rate term
\cite{BoucheronBousquetLugosi2005,HannekeLarsenZhivotovskiy2024}.  Other
results give the optimal order when $L^*$ is close to zero, namely when
$L^*\lesssim(d+\log(1/\delta))/n$.  In this regime, Long's agnostic
one-inclusion graph algorithm \cite{Long1999}, combined with the probability
arguments in \Cref{sec:loo-to-pac}, achieves the optimal bound, as does the
result of Asilis, H{\o}gsgaard, and Velegkas
\cite{AsilisHogsgaardVelegkas2025}.
The resulting bounds have a leading coefficient greater than one on $L^*$,
which is harmless in this regime because the excess is absorbed by the fast
rate term.

We prove the following theorem.

\begin{theorem}[An optimal agnostic PAC learner]\label{thm:pac}
There is a deterministic, generally improper learner that uses neither $L^*$
nor $\delta$ and outputs a classifier
$\widehat h:\mathcal X\to\{-1,+1\}$ such that for every
distribution, every $n\ge1$, and every $0<\delta<1$, with probability at
least $1-\delta$,
\begin{equation}\label{eq:pac-main}
 L(\widehat h)\le L^*+2\cdot10^8\left(
 \sqrt{\frac{L^*(d+\log(48/\delta))}{n}}
 +\frac{d+\log(48/\delta)}{n}\right).
\end{equation}
\end{theorem}

Together with the matching lower bound in \eqref{eq:minimax-lower},
\Cref{thm:pac} determines, up to universal constants, the distribution-free
high probability minimax excess risk bound for binary classification at every
fixed $L^*$.
Our construction remains close to Long's original idea of orienting the entire
Boolean cube in the agnostic one-inclusion graph algorithm and deriving the
corresponding leave-one-out bound.  The key new ingredient is
\Cref{lem:edge-rad}, a class dependent edge
isoperimetric inequality that controls induced edge counts through
approximation errors and projected Rademacher widths.  Once this orientation is available, we
apply suffix averaging with variance control, adapted from Aden-Ali,
Cherapanamjeri, Shetty, and Zhivotovskiy
\cite{AdenAliEtAl2023Optimal}, to rules trained on successive prefixes of the
sample.  A final thresholding step derandomizes the predictor and gives the
desired high probability bound.

\section{Edge isoperimetry and the optimal leave-one-out bound}

This section proves the optimal leave-one-out bound used in
\Cref{thm:pac}.  \Cref{sec:loo-to-pac} converts the resulting leave-one-out
rule into an optimal PAC learner.
We first introduce the cube notation for the one-inclusion analysis.  Let
$P$ be a nonempty finite coordinate set and let $V_P=\{-1,+1\}^P$.  For $v\in V_P$ and
$p\in P$, let $v^{\oplus p}$ be obtained from $v$ by flipping coordinate
$p$.  The Boolean cube $G_P=(V_P,E_P)$ has
\[
 E_P=\bigl\{\{v,v^{\oplus p}\}:v\in V_P,\ p\in P\bigr\}.
\]
For $U\subseteq V_P$, let $E_U=\{e\in E_P:e\subseteq U\}$ be the cube
edges with both endpoints in $U$.  For $v\in V_P$, let
$D_U(v)=\{p\in P:v^{\oplus p}\in U\}$ be the coordinate directions from
$v$ to a neighbor in $U$.  Thus $G_P[U]=(U,E_U)$ and
$\deg_U(v)=|D_U(v)|$ for $v\in U$.  In particular, $G_P[\F]$ is the one-inclusion graph
of a trace $\F\subseteq V_P$ \cite{HLW1994}.  Orienting the full cube with
vertex bounds depending on distance to the trace goes back to Long
\cite[Lemmas~15 and 16]{Long1999} and also appears in later agnostic analyses
\cite{AsilisEtAl2024,DughmiKalayciYork2025}.  The sets $U\subseteq V_P$
will be the cuts in the Hall condition below.

Fix a nonempty $\F\subseteq V_P$ with $\VC(\F)\le d$.  For $u,v\in V_P$,
define their disagreement set by $\DIS(u,v)=\{p\in P:u(p)\ne v(p)\}$.  The Hamming distance from $v$ to
$\F$ is
\[
 \rhoF(v)=\min_{f\in\F}|\DIS(v,f)|.
\]
For every $v,v'\in V_P$ with $|\DIS(v,v')|=1$, the triangle inequality gives
\begin{equation}\label{eq:rho-lipschitz}
 |\rhoF(v)-\rhoF(v')|\le1.
\end{equation}
For $D\subseteq P$, let
\[
\Rad_D(\F)=\E_\varepsilon\max_{f\in\F}
 \sum_{p\in D}\varepsilon_p f(p)
\]
be the unnormalized signed Rademacher width of the projection $\F|_D$.
Keeping the constants in the chaining and covering arguments of Devroye and
Lugosi \cite[Sections~3.2 and 4.3]{DevroyeLugosi2001} gives
\begin{equation}\label{eq:vc-width}
 \Rad_D(\F)\le60\sqrt{d|D|}.
\end{equation}
Up to a universal constant, the same bound follows by combining the VC
covering number bound of Dudley \cite{Dudley1978} with chaining. This is the only point at which the VC dimension enters our proof.  By
contrast, the realizable one-inclusion analysis of Haussler, Littlestone, and
Warmuth uses the VC dimension directly to control induced edge counts
\cite{HLW1994}.

The following lemma is the key connection between induced edge
counts on the Boolean cube, the Hamming distance to $\F$, and projected
Rademacher widths.  We discuss its isoperimetric interpretation and partial
precedents after the proof.

\begin{lemma}[Class dependent edge isoperimetry on the Boolean cube]
\label{lem:edge-rad}
There are weights $w_{v,p}\in[0,1]$, fixed simultaneously for all
$v\in V_P$ and $p\in P$ with the following properties.  The first relation below holds for every
$v\in V_P$ and $p\in P$, and the second for every $v\in V_P$ and
$D\subseteq P$:
\begin{align}
 &w_{v,p}+w_{v^{\oplus p},p}=1,
 \label{eq:fractional-edge}\\
 &\sum_{p\in D}w_{v,p}\le \rhoF(v)+\Rad_D(\F).
 \label{eq:edge-rad}
\end{align}
Consequently, for every $U\subseteq V_P$, the induced subgraph satisfies
the localized edge bound
\begin{equation}\label{eq:edge-cut}
 |E_U|\le\sum_{v\in U}\left(\rhoF(v)+\Rad_{D_U(v)}(\F)\right).
\end{equation}
\end{lemma}

\begin{proof}
Let $T$ contain each coordinate of $P$ independently with probability
$1/2$.  Since $\F$ is finite, fix any order on it.  For each $T\subseteq P$, let $\pi_T(v)$ be
the first member of $\F$ minimizing $|\DIS(v,f)\cap T|$.  In particular,
\begin{equation}\label{eq:consistency}
 v|_T=v'|_T
 \quad\text{implies}\quad
 \pi_T(v)=\pi_T(v').
\end{equation}
Define the weights
\begin{equation}\label{eq:weight}
 w_{v,p}
 =\Prob_T\left(p\in\DIS(v,\pi_T(v))\mid p\notin T\right).
\end{equation}
Clearly $w_{v,p}\in[0,1]$.  Suppose $p\notin T$.  Then
$v|_T=v^{\oplus p}|_T$, so \eqref{eq:consistency} gives
\[
 \pi_T(v)=\pi_T(v^{\oplus p}).
\]
The labelings $v$ and $v^{\oplus p}$ have opposite labels at $p$, while
this common fit has one fixed label there.  Hence
\[
 \one\{p\in\DIS(v,\pi_T(v))\}
 +\one\{p\in\DIS(v^{\oplus p},\pi_T(v^{\oplus p}))\}=1.
\]
Taking conditional expectation given $p\notin T$ proves
\eqref{eq:fractional-edge}. It remains to prove \eqref{eq:edge-rad}.  Fix $v\in V_P$ and
$D\subseteq P$, and set
\[
 \varepsilon_p=
 \begin{cases}
  +1,&p\notin T,\\
  -1,&p\in T.
 \end{cases}
\]
Then $(\varepsilon_p)_{p\in P}$ are independent Rademacher signs.  First,
\[
 \sum_{p\in D}w_{v,p}
 =2\E_T\bigl|\DIS(v,\pi_T(v))\cap D\cap T^c\bigr|.
\]
Indeed, $\Prob_T(p\notin T)=1/2$, so the conditional probability in
\eqref{eq:weight} equals twice the corresponding unconditional probability.
Next, applying the pointwise identity
\[
 |B\cap D\cap T^c|
 =|B\cap D\cap T|
  +\sum_{p\in D}\varepsilon_p\one\{p\in B\}
\]
to $B=\DIS(v,\pi_T(v))$ gives
\begin{align}
 \sum_{p\in D}w_{v,p}
 &=2\E_T\bigl|\DIS(v,\pi_T(v))\cap D\cap T\bigr|
   +2\E_T\sum_{p\in D}\varepsilon_p
       \one\{p\in\DIS(v,\pi_T(v))\}\notag\\
 &\le 2\E_T\bigl|\DIS(v,\pi_T(v))\cap T\bigr|
   +2\E_T\max_{f\in\F}\sum_{p\in D}\varepsilon_p
       \one\{p\in\DIS(v,f)\}.
 \label{ineq:key}
\end{align}
For the first term, we dropped the restriction to $D$; for the second, we
replaced the particular choice $\pi_T(v)\in\F$ by the maximum over
$f\in\F$.  By the definition of $\pi_T(v)$,
\[
 \begin{aligned}
 2\E_T\bigl|\DIS(v,\pi_T(v))\cap T\bigr|
 &=2\E_T\min_{f\in\F}|\DIS(v,f)\cap T|\\
 &\le\min_{f\in\F}2\E_T|\DIS(v,f)\cap T|
 =\rhoF(v).
 \end{aligned}
\]
Meanwhile, since
\[
 2\one\{p\in\DIS(v,f)\}=1-v(p)f(p),
\]
we have
\begin{align*}
 2\E_T\max_{f\in\F}\sum_{p\in D}\varepsilon_p
       \one\{p\in\DIS(v,f)\}
 &=\E_T\sum_{p\in D}\varepsilon_p
   +\E_T\max_{f\in\F}\sum_{p\in D}
       (-\varepsilon_pv(p))f(p)\\
 &=\Rad_D(\F).
\end{align*}
Here $\E_T\sum_{p\in D}\varepsilon_p=0$, and
$(-\varepsilon_pv(p))_{p\in D}$ are again independent Rademacher signs.
Combining the last two displays with \eqref{ineq:key} proves
\eqref{eq:edge-rad}.  Finally, every edge $\{v,v^{\oplus p}\}\in E_U$
contributes
$w_{v,p}+w_{v^{\oplus p},p}=1$ to the incidence sum.  Hence we have
\[
 |E_U|=\sum_{v\in U}\sum_{p\in D_U(v)}w_{v,p}.
\]
Applying \eqref{eq:edge-rad} with $D=D_U(v)$ and summing over $v\in U$
proves \eqref{eq:edge-cut}.
\end{proof}

\begin{remark}[Relation to cube edge isoperimetry]
For every nonempty $U\subseteq V_P$, the edge isoperimetric inequality on
the cube \cite[Theorem~2.39]{ODonnell2014} gives
\[
 |E_U|\le\frac{|U|}{2}\log_2|U|.
\]
For a subcube of dimension $k$, both sides equal $k2^{k-1}$. 
\Cref{lem:edge-rad} gives instead the following localized bound
relative to any $\F\subseteq V_P$:
\[
 |E_U|\le\sum_{v\in U}
 \left(\rhoF(v)+\Rad_{D_U(v)}(\F)\right).
\]
The quantity on the right depends on the position of $U$ relative to $\F$
and on the internal coordinate sets $D_U(v)$.  In general, neither bound dominates
the other.
\end{remark}

\begin{remark}
Several earlier results capture different parts of \Cref{lem:edge-rad}.
\begin{itemize}
\item For every nonempty $U\subseteq\F$, Haussler, Littlestone, and Warmuth
\cite{HLW1994} proved
\[
 |E_U|\le\VC(U)|U|\le d|U|.
\]
The inclusion $U\subseteq\F$ is essential for the second inequality and
leaves an arbitrary $U\subseteq V_P$ uncontrolled by $d$ alone.

\item Long's cube orientation \cite[Lemmas~15 and 16]{Long1999} applies to every
$U\subseteq V_P$ and gives
\[
 |E_U|\le15\sum_{v\in U}\left(d+\rhoF(v)\right).
\]
The construction uses distance shells but does not give a Rademacher width
restricted to a coordinate set.  More importantly, the coefficient $15$ on
$\rhoF(v)$ prevents a leave-one-out bound with coefficient one on the
empirical optimum.

\item Asilis, Devic, Dughmi, Sharan, and Teng
\cite[Appendix~E.3, proof of Proposition~69]{AsilisEtAl2024} use the same
omitted coordinate mechanism to define a fractional orientation.  Hiding
the edge coordinate makes its endpoints conditionally identical, and averaging
gives \eqref{eq:fractional-edge}.  Their argument does
not give \eqref{eq:edge-rad} simultaneously for every $D\subseteq P$.

\item For every $U\subseteq V_P$, Dughmi, Kalayci, and York prove the
all-coordinate bound
\[
 |E_U|\le\sum_{v\in U}\rhoF(v)
 +\frac{|U|}{2}\Rad_P(\F),
\]
with equality for $U=V_P$; see
\cite[Lemmas~8 and 9, and Section~3.4]{DughmiKalayciYork2025}.
Their inequality uses the single full width $\Rad_P(\F)$.
\Cref{lem:edge-rad} instead controls one assignment for every
$D\subseteq P$, allowing the vertexwise choice $D=D_U(v)$.
\end{itemize}
\end{remark}

\Cref{lem:edge-rad} is converted into an orientation by the
Hall argument commonly used in agnostic leave-one-out analyses.  We recall the
required definitions.

An orientation of a finite graph $G=(V,E)$ chooses a head
$\sigma(e)\in e$ for each edge.  Its outdegree at $x$ is
$\operatorname{out}(x;\sigma)
=|\{e\in E:x\in e,\ \sigma(e)\ne x\}|$.
For $W\subseteq V$, write $E_G(W)=\{e\in E:e\subseteq W\}$.  A form of
Hall's theorem states that if integer capacities $c_x\ge0$ satisfy the
following inequality for every $W\subseteq V$,
\[
 |E_G(W)|\le\sum_{x\in W}c_x,
\]
then every edge can be assigned to one endpoint so that, for every $x\in V$,
at most $c_x$ edges are assigned to $x$
\cite[Theorem~2.1.2]{Diestel2017}.  Taking the assigned
endpoint as the tail gives an orientation with outdegree at most $c_x$.
Combining \Cref{lem:edge-rad} with \eqref{eq:vc-width} and this Hall
condition yields the
optimal agnostic leave-one-out bound.

For a labeled sequence
$S=((x_i,y_i))_{i=1}^n\in(\mathcal X\times\{-1,+1\})^n$, a leave-one-out rule
observes all inputs and all labels except $y_i$ when predicting $\widehat y_i$.
We denote
\[
 \widehat L_S^*
 =\frac1n\min_{h\in\Hc}\sum_{i=1}^n\one\{h(x_i)\ne y_i\},\quad \text{and} \quad
 \widehat L_S^{\mathrm{LOO}}
 =\frac1n\sum_{i=1}^n\one\{\widehat y_i\ne y_i\}.
\]

\begin{theorem}[Orientation of the cube and the leave-one-out bound]\label{thm:loo}
Let $P$ be a nonempty finite coordinate set, let $d\ge1$, and let
$\varnothing\ne\F\subseteq V_P$ satisfy $\VC(\F)\le d$.  There is a
deterministic orientation $\sigma$ of the full cube such that, for
every $v\in V_P$,
\begin{equation}\label{eq:orientation-main}
 \operatorname{out}(v;\sigma)
 \le \rhoF(v)+120\sqrt{d\rhoF(v)}+7202d.
\end{equation}
Consequently, for every $n\ge1$ and every input sequence
$x_1,\ldots,x_n$, allowing repetitions, there is a deterministic, generally
improper leave-one-out rule depending only on the unlabeled sequence such that,
for every $(y_1,\ldots,y_n)\in\{-1,+1\}^n$,
\begin{equation}\label{eq:loo-main}
 \widehat L_S^{\mathrm{LOO}}
 \le \widehat L_S^*+120\sqrt{\frac{d\widehat L_S^*}{n}}
      +7202\frac dn.
\end{equation}
\end{theorem}

\begin{proof}
Fix a nonempty $U\subseteq V_P$.  By \eqref{eq:fractional-edge}, we have
\[
 |E_U|=\sum_{v\in U}\sum_{p\in D_U(v)}w_{v,p}.
\]
For every $v\in U$, \eqref{eq:edge-rad} and \eqref{eq:vc-width} give
\[
 \sum_{p\in D_U(v)}w_{v,p}
 \le \rhoF(v)+60\sqrt{d\deg_U(v)}
 \le \rhoF(v)+3601d+\frac{900}{3601}\deg_U(v).
\]
The last step uses the elementary inequality
$\sqrt{ab}\le(a+b)/2$. Relabel the endpoints of an edge $\{v,v^{\oplus p}\}\in E_U$ so that
$\rhoF(v)\le\rhoF(v^{\oplus p})$.
By \eqref{eq:rho-lipschitz} and
$\sqrt{1+x}\le1+x/2$ for $x\ge0$,
\[
 \frac{1/\sqrt{\rhoF(v)+3601d}}
      {1/\sqrt{\rhoF(v^{\oplus p})+3601d}}
 \le\sqrt{1+\frac{1}{3601d}}
 \le1+\frac{1}{7202d}.
\]
Consequently,
\[
 \left|
 \frac{1}{\sqrt{\rhoF(v)+3601d}}
 -\frac{1}{\sqrt{\rhoF(v^{\oplus p})+3601d}}
 \right|
 \le\frac{1}{7202d}\min\left(
 \frac{1}{\sqrt{\rhoF(v)+3601d}},
 \frac{1}{\sqrt{\rhoF(v^{\oplus p})+3601d}}
 \right).
\]
By \eqref{eq:fractional-edge}, it holds that
\begin{align*}
 &\frac{1}{\sqrt{\rhoF(v)+3601d}}
 +\frac{1}{\sqrt{\rhoF(v^{\oplus p})+3601d}}\\
 &=2\left(
 \frac{w_{v,p}}{\sqrt{\rhoF(v)+3601d}}
 +\frac{w_{v^{\oplus p},p}}
 {\sqrt{\rhoF(v^{\oplus p})+3601d}}
 \right)\\
 &\quad +(w_{v^{\oplus p},p}-w_{v,p})
 \left(
 \frac{1}{\sqrt{\rhoF(v)+3601d}}
 -\frac{1}{\sqrt{\rhoF(v^{\oplus p})+3601d}}
 \right)\\
 &\le\left(2+\frac{1}{7202d}\right)
 \left(
 \frac{w_{v,p}}{\sqrt{\rhoF(v)+3601d}}
 +\frac{w_{v^{\oplus p},p}}
 {\sqrt{\rhoF(v^{\oplus p})+3601d}}
 \right).
\end{align*}
The last step uses $|w_{v^{\oplus p},p}-w_{v,p}|\le1$ and the fact that
the weighted average is at least the smaller reciprocal.
Summing this inequality over $E_U$ and applying the preceding bound
gives
\begin{align*}
 \sum_{v\in U}\frac{\deg_U(v)}{\sqrt{\rhoF(v)+3601d}}
 &\le\left(2+\frac{1}{7202d}\right)
 \sum_{v\in U}
 \frac{\sum_{p\in D_U(v)}w_{v,p}}
 {\sqrt{\rhoF(v)+3601d}}\\
 &\le\left(2+\frac{1}{7202d}\right)
 \sum_{v\in U}\left(
 \sqrt{\rhoF(v)+3601d}
 +\frac{900}{3601}
 \frac{\deg_U(v)}{\sqrt{\rhoF(v)+3601d}}
 \right).
\end{align*}
Moving the last term to the left gives
\[
 \left(1-\frac{900}{3601}
       \left(2+\frac{1}{7202d}\right)\right)
 \sum_{v\in U}\frac{\deg_U(v)}{\sqrt{\rhoF(v)+3601d}}\le
 \left(2+\frac{1}{7202d}\right)
 \sum_{v\in U}\sqrt{\rhoF(v)+3601d}.
\]
For $d\ge1$,
$
 \frac{2+1/(7202d)}
 {1-(900/3601)(2+1/(7202d))}<4.
$
Dividing the preceding inequality therefore gives
\[
 \sum_{v\in U}\frac{\deg_U(v)}{\sqrt{\rhoF(v)+3601d}}
 <4\sum_{v\in U}\sqrt{\rhoF(v)+3601d}.
\]
Cauchy-Schwarz inequality now yields
\[
 \sum_{v\in U}\sqrt{\deg_U(v)}
 \le
 \left(\sum_{v\in U}
 \frac{\deg_U(v)}{\sqrt{\rhoF(v)+3601d}}\right)^{1/2}
 \left(\sum_{v\in U}\sqrt{\rhoF(v)+3601d}\right)^{1/2}<2\sum_{v\in U}\sqrt{\rhoF(v)+3601d}.
\]
Finally, we have
\[
 |E_U|
 \le\sum_{v\in U}
 \left(\rhoF(v)+60\sqrt{d\deg_U(v)}\right)
 <\sum_{v\in U}
 \left(\rhoF(v)+120\sqrt{d\left(\rhoF(v)+3601d\right)}\right).
\]
The empty set satisfies the same bound trivially.  Hence Hall's theorem,
with integer capacity
$\lceil\rhoF(v)+120\sqrt{d(\rhoF(v)+3601d)}\rceil$ at $v$, gives an orientation
such that, using $\lceil x\rceil<x+1$ and
$\sqrt{x+y}\le\sqrt{x}+\sqrt{y}$,
\[
 \operatorname{out}(v;\sigma)
<\rhoF(v)+120\sqrt{d\rhoF(v)}+120\sqrt{3601}\,d+1
 <\rhoF(v)+120\sqrt{d\rhoF(v)}+7202d.
\]
Choose the first feasible orientation in a fixed ordering of the finitely
many orientations.  This proves \eqref{eq:orientation-main}.

For the leave-one-out statement, apply the orientation result with
$P=\{1,\ldots,n\}$ to the trace
\[
 \F=\bigl\{(h(x_1),\ldots,h(x_n)):h\in\Hc\bigr\}.
\]
Write $v_y=(y_1,\ldots,y_n)\in V_P$. This edge prediction convention is due to Long \cite{Long1999}; see also the
general agnostic one-inclusion formalism of Asilis et
al.~\cite{AsilisEtAl2024}.

\begin{algorithm}[H]
\caption{The agnostic leave-one-out rule.}
\label{alg:loo}
\textbf{Inputs:} The orientation $\sigma$, a held-out index
$i\in\{1,\ldots,n\}$, and the observed labels $(y_j)_{j\ne i}$.\\
\textbf{Output:} A prediction $\widehat y_i\in\{-1,+1\}$.
\vspace{0.5em}

For the held-out coordinate $i$, predict as follows.
\begin{algorithmic}[1]
 \STATE Fix coordinate $j$ at $y_j$ for every $j\ne i$.
 \STATE Complete coordinate $i$ with each of its two signs; the two
 completions form a cube edge.
 \STATE Set $\widehat y_i$ equal to coordinate $i$ of the head selected by
 $\sigma$.
\end{algorithmic}
\end{algorithm}

By construction, the predictor is wrong at coordinate $i$ exactly when the true
completion is the tail of this edge.  Hence
\[
 n\widehat L_S^{\mathrm{LOO}}=\operatorname{out}(v_y;\sigma).
\]
Since $\rhoF(v_y)=n\widehat L_S^*$ and $\VC(\F)\le d$, dividing
\eqref{eq:orientation-main} by $n$ proves \eqref{eq:loo-main}.
\end{proof}

\section{From leave-one-out to PAC bounds}\label{sec:loo-to-pac}

Let $P$ be a distribution on $\mathcal X\times\{-1,+1\}$.  We identify a
randomized binary predictor with its measurable conditional mean
$g:\mathcal X\to[-1,1]$: on input $x$, it predicts $+1$ with probability
$(1+g(x))/2$.  We write $\ell(a,y)=|a-y|/2$ and extend the notation $L$ to
the risk of such a score:
\[
 L(g)=\E_{(X,Y)\sim P}\ell(g(X),Y)
 =\frac12\E_{(X,Y)\sim P}|g(X)-Y|.
\]
For $\{-1,+1\}$-valued $g$, this agrees with the usual binary
classification risk.
If $g$ depends on the data, the expectation is over an independent test
point conditional on the training sample.

\subsection{Symmetrization of the agnostic one-inclusion-graph algorithm}
The orientation from \Cref{thm:loo} need not be invariant under permutations
of the sample, so we symmetrize its predictions.  Given a labeled sample $S$ and
an unlabeled query $x$, append $x$, uniformly permute the resulting indexed
points, and apply the corresponding deterministic leave-one-out rule while hiding the
coordinate occupied by $x$.  Let
$q(x;S)$ be the expectation of the resulting signed prediction over this
permutation.  The score $q(x;S)\in[-1,1]$ is deterministic and symmetric in
the observations of $S$.

For a labeled sample $S=((x_i,y_i))_{i=1}^m$, write $S^{-i}$ for the
sample with observation $i$ removed.  Averaging
\eqref{eq:loo-main} over all permutations gives, for every $S$,
\begin{equation}\label{eq:symmetric-loo}
 \frac1{2m}\sum_{i=1}^m |q(x_i;S^{-i})-y_i|
 \le \widehat L_S^*+120\sqrt{\frac{d\widehat L_S^*}{m}}
       +7202\frac dm.
\end{equation}

\subsection{A randomized learner by suffix averaging}

The argument below adapts the reverse and forward martingale proof of
Aden-Ali, Cherapanamjeri, Shetty, and Zhivotovskiy
\cite{AdenAliEtAl2023Optimal} from the realizable to the agnostic
setting.  Dughmi, Kalayci, and York \cite{DughmiKalayciYork2025} gave a
related conversion for bounded agnostic losses.  Here we retain the
comparator dependent variance needed for a bound uniform in $L^*$ and avoid
the additional logarithmic factor in their analysis.  Even in the realizable
case, some one-inclusion rules that are optimal in expectation have tails no
better than Markov's inequality \cite{AdenAliEtAl2023Negative}.  Thus an
orientation result of \Cref{thm:loo} alone does not give the optimal
dependence on $\delta$.  Suffix averaging gives the required confidence
dependence.

Fix an integer $k\ge1$.  Let $Z_i=(X_i,Y_i)$, $i=1,\ldots,2k$, be
i.i.d.\ observations and write
$S_{\le t}=(Z_1,\ldots,Z_t)$.  For $k\le t\le2k-1$, set
$q_t(x)=q(x;S_{\le t})$, where $q$ is the symmetrized leave-one-out score
satisfying \eqref{eq:symmetric-loo}.
The randomized predictor averages the scores $q_t$ over the suffix of
training times $t=k,\ldots,2k-1$.  Its conditional mean is
\begin{equation}\label{eq:suffix-average}
 \widehat p(x)=\frac1k\sum_{t=k}^{2k-1}q_t(x).
\end{equation}
Thus it predicts $+1$ with probability $(1+\widehat p(x))/2$ and, by
linearity,
$L(\widehat p)=k^{-1}\sum_{t=k}^{2k-1}L(q_t)$.  Only
$Z_1,\ldots,Z_{2k-1}$ enter the construction of $\widehat p$.  The additional
observation $Z_{2k}$ is the fresh point used to analyze $q_{2k-1}$ and does
not enter the learner.  Derandomization is deferred to the next subsection.

For every $h\in\Hc$, linearity gives
\begin{equation}\label{eq:risk-target}
 L(\widehat p)-L(h)
 =\frac1k\sum_{t=k}^{2k-1}\bigl(L(q_t)-L(h)\bigr).
\end{equation}

As in \cite{AdenAliEtAl2023Optimal}, the reverse martingale lemma below uses
permutation invariance to apply \eqref{eq:symmetric-loo} along a random
deletion sequence and control the realized excess loss on the held-out
observations.  The forward martingale lemma compares this sequential excess
loss with its conditional mean, the population excess risk in
\eqref{eq:risk-target}.  Together they yield the randomized guarantee in
\Cref{thm:suffix-averaging}.

\begin{lemma}[Reverse martingale]\label{lem:reverse}
Let $Z_i=(X_i,Y_i)$, $1\le i\le2k$, be i.i.d.\ observations, and let
$q_t(x)=q(x;(Z_1,\ldots,Z_t))$, where $q$ is the deterministic symmetric
score satisfying \eqref{eq:symmetric-loo}.  Fix $h\in\Hc$ and, for
$k\le t\le2k-1$, let
\[
 \Delta_{t+1,h}
 =\ell(q_t(X_{t+1}),Y_{t+1})
  -\one\{h(X_{t+1})\ne Y_{t+1}\}.
\]
For every $0<\delta\le1/3$, with probability at least $1-2\delta$,
\begin{equation}\label{eq:reverse-bound}
 \frac1k
 \max\left(
 \sum_{m=k+1}^{2k}\Delta_{m,h},0
 \right)
 \le
 8200\left(
 \sqrt{\frac{L(h)(d+\log(1/\delta))}{k}}
 +\frac{d+\log(1/\delta)}{k}
 \right).
\end{equation}
\end{lemma}

\begin{lemma}[Forward martingale]\label{lem:forward}
Under the assumptions and notation of \Cref{lem:reverse}, for every
$0<\delta\le1/3$, with probability at least $1-\delta$,
\begin{equation}\label{eq:forward-bound}
 \frac1k\sum_{t=k}^{2k-1}\bigl(L(q_t)-L(h)\bigr)
 \le
 \frac2k
 \max\left(
 \sum_{m=k+1}^{2k}\Delta_{m,h},0
 \right)
 +5\sqrt{\frac{L(h)\log(1/\delta)}{k}}
 +10\frac{\log(1/\delta)}k.
\end{equation}
\end{lemma}

\begin{theorem}[The optimal risk bound for the randomized predictor]\label{thm:suffix-averaging}
The predictor $\widehat p:\mathcal X\to[-1,1]$ in
\eqref{eq:suffix-average} is constructed from $2k-1$ i.i.d.\ observations and
depends on neither $L^*$ nor $\delta$.  For every $0<\delta<1$, with
probability at least $1-\delta$,
\begin{equation}\label{eq:suffix-averaging-bound}
 L(\widehat p)
 \le L^*
 +17000\left(
 \sqrt{\frac{L^*(d+\log(3/\delta))}{k}}
 +\frac{d+\log(3/\delta)}{k}
 \right).
\end{equation}
\end{theorem}

\begin{proof}[Proof of \Cref{thm:suffix-averaging}]
Choose $h^*\in\Hc$ with $L(h^*)=L^*$.  Apply
\Cref{lem:reverse,lem:forward} with confidence $\delta/3$.  A union bound
shows that both conclusions hold simultaneously with probability at least
$1-\delta$.  On this event,
\begin{align*}
 L(\widehat p)-L^*
 &=\frac1k
 \sum_{t=k}^{2k-1}\bigl(L(q_t)-L^*\bigr)
 \\
 &\le
 \frac2k
 \max\left(
 \sum_{m=k+1}^{2k}\Delta_{m,h^*},0
 \right)
 +5\sqrt{\frac{L^*\log(3/\delta)}{k}}
 +10\frac{\log(3/\delta)}k
 \\
 &\le
 16400\left(
 \sqrt{\frac{L^*(d+\log(3/\delta))}{k}}
 +\frac{d+\log(3/\delta)}{k}
 \right)
 +5\sqrt{\frac{L^*\log(3/\delta)}{k}}
 +10\frac{\log(3/\delta)}k
 \\
 &<
 17000\left(
 \sqrt{\frac{L^*(d+\log(3/\delta))}{k}}
 +\frac{d+\log(3/\delta)}{k}
 \right).
\end{align*}
The claim follows.
\end{proof}

\begin{proof}[Proof of \Cref{lem:reverse}]
Fix $h\in\Hc$ and $0<\delta\le1/3$.  The
independent permutation below exploits the symmetry of $q$.  Conditional on
the observations, exposing the permutation backward makes the next removed
observation uniform among those that remain.  Draw i.i.d.\ observations
$W_i=(\widetilde X_i,\widetilde Y_i)$, $1\le i\le2k$, and an independent
uniform permutation $\pi$ of $\{1,\ldots,2k\}$.  Set
$Z_m=W_{\pi(m)}$.

For $k\le m\le2k$, let $\mathcal R_m$ be the $\sigma$-algebra generated by
the complete sample and the permutation indices already exposed at positions
$2k,\ldots,m+1$:
\[
 \mathcal R_m
 =\sigma\bigl(W_1,\ldots,W_{2k},\pi(j):m<j\le2k\bigr).
\]
Thus $\mathcal R_m$ reveals which $m$ indices remain, but not their order.
Let $I_m$ be the set of these indices and write
$B_m=(W_i)_{i\in I_m}$ in any fixed order.
For $i\in I_m$, let $B_m^{-i}$ be obtained from $B_m$ by deleting $W_i$.
Conditional on $\mathcal R_m$, the index $\pi(m)$ is uniform on $I_m$.
Because $q$ is symmetric, if $\pi(m)=i$, then the first $m-1$ observations
give the same prediction as $q(\,\cdot\,;B_m^{-i})$.

For $i\in I_m$, let
\[
 \Delta_{i,m,h}
 =\ell\bigl(q(\widetilde X_i;B_m^{-i}),\widetilde Y_i\bigr)
  -\one\{h(\widetilde X_i)\ne\widetilde Y_i\}.
\]
It follows that $\Delta_{m,h}=\Delta_{\pi(m),m,h}$ and
\[
 \E[\Delta_{m,h}\mid\mathcal R_m]
 =\frac1m\sum_{i\in I_m}\Delta_{i,m,h}.
\]
After $\pi(m)$ is revealed, $\Delta_{m,h}$ is measurable with respect to
$\mathcal R_{m-1}$.  Hence, as $m$ runs from $2k$ down to $k+1$,
\[
 \Delta_{m,h}-\E[\Delta_{m,h}\mid\mathcal R_m]
\]
forms a martingale difference sequence with respect to the reverse
filtration.

The empirical optimum on $B_m$ is at most the empirical loss of $h$.
Applying \eqref{eq:symmetric-loo} to $B_m$ and subtracting the empirical loss
of $h$ gives
\begin{equation}
 \E[\Delta_{m,h}\mid\mathcal R_m]
 \le
 120\frac{\sqrt{d\sum_{i\in I_m}
 \one\{h(\widetilde X_i)\ne\widetilde Y_i\}}}{m}
 +7202\frac dm
 \le
 120\frac{\sqrt{d\sum_{i=1}^{2k}
 \one\{h(\widetilde X_i)\ne\widetilde Y_i\}}}{k}
 +7202\frac dk,
 \label{eq:reverse-mean-bound}
\end{equation}
where the last inequality uses $m\ge k$ and bounds the sum over $I_m$ by
the sum over the complete sample.

For the conditional second moment, applying $(a-b)^2\le a+b$, valid for
$a\in[0,1]$ and $b\in\{0,1\}$, to the two losses, averaging over the uniform
choice of $i\in I_m$, and using \eqref{eq:symmetric-loo} again gives
\begin{align}
 \E[\Delta_{m,h}^2\mid\mathcal R_m]
 &\le
 \frac2m\sum_{i\in I_m}
 \one\{h(\widetilde X_i)\ne\widetilde Y_i\}
 +120\frac{\sqrt{d\sum_{i\in I_m}
 \one\{h(\widetilde X_i)\ne\widetilde Y_i\}}}{m}
 +7202\frac dm
 \notag\\
 &\le
 \frac2k\sum_{i=1}^{2k}
 \one\{h(\widetilde X_i)\ne\widetilde Y_i\}
 +120\frac{\sqrt{d\sum_{i=1}^{2k}
 \one\{h(\widetilde X_i)\ne\widetilde Y_i\}}}{k}
 +7202\frac dk.
 \label{eq:reverse-second-moment}
\end{align}
The number of errors made by $h$ on the complete sample is distributed as
$\operatorname{Bin}(2k,L(h))$.  Bernstein's inequality and
$2\sqrt{kL(h)\log(1/\delta)}\le kL(h)+\log(1/\delta)$ imply that, with
probability at least $1-\delta$,
\begin{equation}\label{eq:binomial-bound}
 \frac1k\sum_{i=1}^{2k}
 \one\{h(\widetilde X_i)\ne\widetilde Y_i\}
 \le 3\left(L(h)+\frac{\log(1/\delta)}k\right).
\end{equation}
On the event in \eqref{eq:binomial-bound},
\begin{align}
 120\frac{\sqrt{d\sum_{i=1}^{2k}
 \one\{h(\widetilde X_i)\ne\widetilde Y_i\}}}{k}
 +7202\frac dk
 &\le
 120\sqrt{\frac{3d(L(h)+\log(1/\delta)/k)}{k}}
 +7202\frac dk
 \notag\\
 &\le
 8000\left(
 \sqrt{\frac{L(h)(d+\log(1/\delta))}{k}}
 +\frac{d+\log(1/\delta)}{k}
 \right).
 \label{eq:deterministic-remainder-bound}
\end{align}

Conditionally on $W_1,\ldots,W_{2k}$, we apply Freedman's inequality
\cite{Freedman1975} to the preceding martingale differences, with $m$ read
in decreasing order.  By \eqref{eq:reverse-second-moment}, their predictable
quadratic variation satisfies
\[
 \sum_{m=k+1}^{2k}
 \E\left[\left(\Delta_{m,h}
 -\E[\Delta_{m,h}\mid\mathcal R_m]\right)^2\middle|\mathcal R_m\right]
 \le 2\sum_{i=1}^{2k}
 \one\{h(\widetilde X_i)\ne\widetilde Y_i\}
 +120\sqrt{d\sum_{i=1}^{2k}
 \one\{h(\widetilde X_i)\ne\widetilde Y_i\}}
 +7202d,
\]
and
$|\Delta_{m,h}-\E[\Delta_{m,h}\mid\mathcal R_m]|\le2$.
Moreover, summing \eqref{eq:reverse-mean-bound} over
$m=k+1,\ldots,2k$ gives
\[
 \sum_{m=k+1}^{2k}\E[\Delta_{m,h}\mid\mathcal R_m]
 \le120\sqrt{d\sum_{i=1}^{2k}
 \one\{h(\widetilde X_i)\ne\widetilde Y_i\}}+7202d.
\]
Therefore, conditionally on $W_1,\ldots,W_{2k}$, with probability at least
$1-\delta$,
\begin{align}
 \frac1k\sum_{m=k+1}^{2k}\Delta_{m,h}
 &\le
 120\frac{\sqrt{d\sum_{i=1}^{2k}
 \one\{h(\widetilde X_i)\ne\widetilde Y_i\}}}{k}
 +7202\frac dk
 \notag\\
 &\quad+
 \sqrt{
 \left(
 \frac4k\sum_{i=1}^{2k}
 \one\{h(\widetilde X_i)\ne\widetilde Y_i\}
 +240\frac{\sqrt{d\sum_{i=1}^{2k}
 \one\{h(\widetilde X_i)\ne\widetilde Y_i\}}}{k}
 +14404\frac dk
 \right)
 \frac{\log(1/\delta)}k
 }
 \notag\\
 &\quad
 +\frac{2\log(1/\delta)}{3k}.
 \label{eq:reverse-freedman}
\end{align}
It remains to simplify the square root term.  On the event in
\eqref{eq:binomial-bound}, \eqref{eq:deterministic-remainder-bound} gives
\begin{align*}
 &\left(
 \frac4k\sum_{i=1}^{2k}
 \one\{h(\widetilde X_i)\ne\widetilde Y_i\}
 +240\frac{\sqrt{d\sum_{i=1}^{2k}
 \one\{h(\widetilde X_i)\ne\widetilde Y_i\}}}{k}
 +14404\frac dk
 \right)
 \frac{\log(1/\delta)}k
 \\
 &\qquad\le
 12\frac{L(h)\log(1/\delta)}{k}
 +12\left(\frac{\log(1/\delta)}k\right)^2
 \\
 &\qquad\quad{}+16000\left(
 \sqrt{\frac{L(h)(d+\log(1/\delta))}{k}}
 +\frac{d+\log(1/\delta)}{k}
 \right)\frac{\log(1/\delta)}k
 \\
 &\qquad\le
 16024\left(
 \sqrt{\frac{L(h)(d+\log(1/\delta))}{k}}
 +\frac{d+\log(1/\delta)}{k}
 \right)^2.
\end{align*}
The last inequality follows because both
$\sqrt{L(h)\log(1/\delta)/k}$ and $\log(1/\delta)/k$ are bounded by the
parenthesized expression.
Substituting these bounds into \eqref{eq:reverse-freedman} gives
\begin{align*}
 \frac1k\sum_{m=k+1}^{2k}\Delta_{m,h}
 &\le
 8000\left(
 \sqrt{\frac{L(h)(d+\log(1/\delta))}{k}}
 +\frac{d+\log(1/\delta)}{k}
 \right)
 \\
 &\quad+
 \sqrt{16024}\left(
 \sqrt{\frac{L(h)(d+\log(1/\delta))}{k}}
 +\frac{d+\log(1/\delta)}{k}
 \right)
 +\frac{2\log(1/\delta)}{3k}
 \\
 &<
 8200\left(
 \sqrt{\frac{L(h)(d+\log(1/\delta))}{k}}
 +\frac{d+\log(1/\delta)}{k}
 \right).
\end{align*}
The event in \eqref{eq:binomial-bound} has probability at least $1-\delta$,
and averaging
the conditional Freedman bound shows that \eqref{eq:reverse-freedman} holds
with probability at least $1-\delta$.  Thus both bounds hold simultaneously
with probability at least $1-2\delta$.  On their intersection, taking the
positive part in the preceding display gives \eqref{eq:reverse-bound}.  Since
$(Z_1,\ldots,Z_{2k})$ has the original i.i.d.\ law, the proof follows.
\end{proof}

\begin{proof}[Proof of \Cref{lem:forward}]
Fix $h\in\Hc$ and $0<\delta\le1/3$.  Conditionally on $S_{\le t}$, the prediction
$q_t$ is fixed and $Z_{t+1}$ is an independent observation.  Hence
\[
 \E[\Delta_{t+1,h}\mid S_{\le t}]
 =L(q_t)-L(h).
\]
Since $(a-b)^2\le a+b$ for $a,b\in[0,1]$, applying this inequality to the
losses of $q_t$ and $h$ gives
\begin{equation}\label{eq:forward-variance}
 \operatorname{Var}(\Delta_{t+1,h}\mid S_{\le t})
 \le \E[\Delta_{t+1,h}^2\mid S_{\le t}]
 \le L(q_t)+L(h)=L(q_t)-L(h)+2L(h).
\end{equation}
Therefore
\[
 \xi_{t+1,h}
 =L(q_t)-L(h)-\Delta_{t+1,h}
\]
is a martingale difference with respect to the sample prefixes.  It
satisfies $\xi_{t+1,h}\le2$ and, by \eqref{eq:forward-variance},
\[
 \E[\xi_{t+1,h}^2\mid S_{\le t}]
 \le L(q_t)-L(h)+2L(h).
\]
Fix $0<\lambda<3/2$.  The Bernstein moment generating bound
gives
\[
 \E\left[
 \exp\left(
 \lambda\xi_{t+1,h}
 -\frac{\lambda^2}{2(1-2\lambda/3)}
 \bigl(L(q_t)-L(h)+2L(h)\bigr)
 \right)
 \mathrel{}\middle|\mathrel{}S_{\le t}
 \right]
 \le1.
\]
Iterating this inequality over $t$ and applying Markov's
inequality shows that, with probability at least $1-\delta$,
\begin{equation}\label{eq:forward-mgf}
\sum_{t=k}^{2k-1}\bigl(L(q_t)-L(h)\bigr)
 -\sum_{m=k+1}^{2k}\Delta_{m,h}
\le
 \frac{\lambda}{2(1-2\lambda/3)}
 \left(
 \sum_{t=k}^{2k-1}\bigl(L(q_t)-L(h)\bigr)
 +2kL(h)
 \right)
 +\frac{\log(1/\delta)}{\lambda}.
\end{equation}
Choose
$\lambda=\min\left(\frac12,\sqrt{\frac{\log(1/\delta)}{kL(h)}}\right)$
when $L(h)>0$, and $\lambda=1/2$ when $L(h)=0$.  Since $\lambda\le1/2$,
simple algebra gives
\begin{equation}\label{eq:forward-prechoice}
 \frac1k\sum_{t=k}^{2k-1}\bigl(L(q_t)-L(h)\bigr)
 \le
 \frac2k
 \max\left(
 \sum_{m=k+1}^{2k}\Delta_{m,h},0
 \right)+3\lambda L(h)
 +2\frac{\log(1/\delta)}{\lambda k}.
\end{equation}
If $\lambda=\sqrt{\log(1/\delta)/(kL(h))}$, the last two terms in
\eqref{eq:forward-prechoice} equal
\[
 3\sqrt{\frac{L(h)\log(1/\delta)}{k}}
 +2\sqrt{\frac{L(h)\log(1/\delta)}{k}}
 =5\sqrt{\frac{L(h)\log(1/\delta)}{k}}.
\]
If $\lambda=1/2$, then $L(h)\le4\log(1/\delta)/k$, and
\[
 3\lambda L(h)+2\frac{\log(1/\delta)}{\lambda k}
 =\frac32L(h)+4\frac{\log(1/\delta)}k
 \le10\frac{\log(1/\delta)}k.
\]
This proves \eqref{eq:forward-bound}.
\end{proof}

\subsection{A deterministic binary predictor}

\Cref{thm:suffix-averaging} gives the desired bound for a randomized
predictor, while \Cref{thm:pac} requires a deterministic binary classifier.
In the realizable case, suffix averaging can be made binary by
majority vote at a factor of at most two in risk
\cite{AdenAliEtAl2023Optimal}.  Such a factor is harmless when $L^*=0$, but
would lose the coefficient one on $L^*$ in the agnostic case.  We instead
choose a threshold on an independent validation sample.

The final step is empirical risk minimization over the class obtained by
thresholding the fixed score.  For any reference threshold, the hypotheses
above and below it form two nested families.  This structure gives the local
entropy bound below without an additional logarithmic factor in the sample
size.  Applying this bound in
\cite[Theorem~3.3]{BartlettBousquetMendelson2005} gives the next lemma with
the displayed constants.

\begin{lemma}
\label{lem:threshold-validation}
Let $p:\mathcal X\to[-1,1]$ be measurable and fixed independently of an
i.i.d.\ validation sample $Z_i=(X_i,Y_i)$, $1\le i\le k$, where $k\ge1$.
Consider the binary threshold class
\[
 \mathcal T_p
 =\bigl\{x\mapsto2\one\{p(x)\ge u\}-1:-1\le u\le1\bigr\}
 \cup\{x\mapsto-1\}.
\]
If $\widehat h$ is an empirical risk minimizer over $\mathcal T_p$, then, for
every $0<\delta<1$, with probability at least $1-\delta$,
\[
 L(\widehat h)\le \inf_{h\in\mathcal T_p}L(h)+223000\left(
 \sqrt{\frac{\inf_{h\in\mathcal T_p}L(h)\log(12/\delta)}k}
 +\frac{\log(12/\delta)}k\right).
\]
\end{lemma}

\begin{proof}
Assume first that the infimum is attained and let $h^*$ be a minimizer.
Otherwise, take $h^*$ with risk arbitrarily close to the infimum and let the
approximation error tend to zero.
For $h\in\mathcal T_p$, set
\[
 g_h(x,y)=\one\{h(x)\ne y\}-\one\{h^*(x)\ne y\}.
\]
Then
$\E(g_h(X,Y)+2L(h^*))=L(h)+L(h^*)$ and
\begin{align*}
 \operatorname{Var}(g_h(X,Y)+2L(h^*))
 &=\operatorname{Var}(g_h(X,Y))
 \le \E g_h(X,Y)^2 \\
 &=\Prob(h(X)\ne h^*(X))
 \le L(h)+L(h^*)
 =\E(g_h(X,Y)+2L(h^*)).
\end{align*}
Order the threshold chain pointwise.  Relative to $h^*$, the chain has two branches.
The upper branch consists of the hypotheses $h$ satisfying
$h(x)\ge h^*(x)$ for every $x\in\mathcal X$, and the lower branch consists
of those satisfying $h(x)\le h^*(x)$ for every $x\in\mathcal X$.  Within
either branch the sets $\DIS(h,h^*)$ are nested,
while sets from opposite branches are disjoint.  For every $(x,y)$,
$g_h(x,y)=-y\one_{\DIS(h,h^*)}(x)$ on the upper branch and
$g_h(x,y)=y\one_{\DIS(h,h^*)}(x)$ on the lower branch.

For $s>0$, let $\mathcal G_s=\{g_h:\E g_h(X,Y)^2\le s\}$.  Let $A_s^+$ be
the union of $\DIS(h,h^*)$ over the upper branch with
$g_h\in\mathcal G_s$, and
define $A_s^-$ analogously over the lower branch.  For every such $h$,
$\Prob(X\in\DIS(h,h^*))=\E g_h(X,Y)^2\le s$.  For these nested threshold sets, the
probability of each union is the supremum of the probabilities of its
members.  Thus $\Prob(X\in A_s^+)\le s$ and
$\Prob(X\in A_s^-)\le s$.  Since sets
from opposite branches are disjoint, $A_s^+\cap A_s^-=\varnothing$.  Define
\[
 q=\frac1k\sum_{i=1}^k\one\{X_i\in A_s^+\cup A_s^-\},
\]
and let $t_h$ be the empirical mass of $\DIS(h,h^*)$, with positive sign for
the upper branch and negative sign for the lower branch.  Since the
disagreement sets are nested within each branch and disjoint across branches,
for $g_h,g_{h'}\in\mathcal G_s$,
\[
 \frac1k\sum_{i=1}^k
 \bigl(g_h(Z_i)-g_{h'}(Z_i)\bigr)^2=|t_h-t_{h'}|.
\]
The coordinates $t_h$ lie in an interval of length at most $q$.  Hence, for
$0<u\le\sqrt q$, the class $\mathcal G_s$ has an empirical $L_2$ cover of
radius $u$ and size at most $1+q/u^2$.  Applying the chaining argument
conditionally to the Rademacher process and retaining the numerical constant
in the proof of
\cite[Section~3.2]{DevroyeLugosi2001} gives
\begin{align*}
 \E_{\varepsilon}\sup_{g\in\mathcal G_s}\frac1k\sum_{i=1}^k
 \varepsilon_i g(Z_i) \le \frac{12}{\sqrt k}\int_0^{\sqrt q}
 \sqrt{\log\left(1+\frac{q}{u^2}\right)}\,du
 \le12\sqrt{\frac{(2+\log 2)q}{k}}.
\end{align*}
Taking expectation over the validation sample and using
$\E q=\Prob(X\in A_s^+\cup A_s^-)\le2s$ shows that the last display is less
than $28\sqrt{s/k}$.  The common translation by $2L(h^*)$ has zero expected
Rademacher sum.  Finally, apply the first part of
\cite[Theorem~3.3]{BartlettBousquetMendelson2005} to the class
$\{g_h+2L(h^*):h\in\mathcal T_p\}$, with
$T(g_h+2L(h^*))=\E g_h(X,Y)^2$ and $B=1$.  The subroot function
$s\mapsto28\sqrt{s/k}$ has fixed point $784/k$.  Empirical optimality gives
\[
 \frac1k\sum_{i=1}^k\bigl(g_{\widehat h}(Z_i)+2L(h^*)\bigr)
 \le\frac1k\sum_{i=1}^k\bigl(g_{h^*}(Z_i)+2L(h^*)\bigr)
 =2L(h^*).
\]
Since the range of this class has length two, the cited theorem shows that,
for every fixed $K>1$, with probability at least $1-\delta$,
\[
 L(\widehat h)-L(h^*)
 \le \frac{2L(h^*)}{K-1}
 +\frac{K(551936+26b)+22b}{k},
 \qquad b=\log(12/\delta).
\]
For $L(h^*)>0$, choose
$K=1+\sqrt{2L(h^*)k/(551936+26b)}$; for $L(h^*)=0$, take
$K=1+(551936+26b)^{-1}$.  Since $b\ge\log 12$ and
$551936+26b<222142b$, substitution gives the stated bound.
\end{proof}

\begin{proof}[Proof of \Cref{thm:pac}]
For $n\ge6$, let $k=\lfloor n/3\rfloor$.  Form the $[-1,1]$-valued
predictor $\widehat p$ in \eqref{eq:suffix-average} from observations
$1,\ldots,2k-1$, and use the next $k$ observations,
$2k,\ldots,3k-1$, as an independent validation block.  List the distinct
values $\widehat p(X_i)$ on the validation block in increasing order.  For
each listed value $u$, include the classifier
$x\mapsto2\one\{\widehat p(x)\ge u\}-1$, together with the rule identically
equal to $-1$, and choose any empirical risk
minimizer $\widehat h$ from this finite list, with ties resolved by a fixed
rule. 
Consequently, $\widehat h$ is an empirical risk minimizer over
$\mathcal T_{\widehat p}$.  With the fixed rule for resolving ties,
$\widehat h$ is a deterministic function of the sample and is the classifier
output by the learner. For every $x\in\mathcal X$ and $y\in\{-1,+1\}$,
\[
 \frac12\int_{-1}^1
 \one\{2\,\one\{\widehat p(x)\ge u\}-1\ne y\}\,du
 =\frac12|\widehat p(x)-y|.
\]
Integrating this identity and applying Fubini's theorem gives
\[
 \frac12\int_{-1}^1
 L\bigl(x\mapsto2\one\{\widehat p(x)\ge u\}-1\bigr)\,du
 =L(\widehat p).
\]
Therefore, we have
$\inf_{h\in\mathcal T_{\widehat p}}L(h)\le L(\widehat p)$.

Condition on the first $2k-1$ observations, so that $\widehat p$ is fixed and
the validation block remains i.i.d.  Applying
\Cref{lem:threshold-validation} with confidence $\delta/4$, and
using the preceding inequality, gives
\[
 L(\widehat h)\le L(\widehat p)
 +223000\left(
 \sqrt{\frac{L(\widehat p)\log(48/\delta)}{k}}
 +\frac{\log(48/\delta)}k
 \right).
\]
We apply \Cref{thm:suffix-averaging} with confidence $\delta/16$.  On the event
supplied by this theorem,
\[
 L(\widehat p)
 \le L^*+17000\left(
 \sqrt{\frac{L^*(d+\log(48/\delta))}{k}}
 +\frac{d+\log(48/\delta)}{k}
 \right),
\]
and the inequality $\sqrt{a+b}\le\sqrt a+\sqrt b$ gives
\begin{align*}
 \sqrt{\frac{L(\widehat p)\log(48/\delta)}{k}}
 &\le\sqrt{\frac{L^*\log(48/\delta)}{k}}
 \\
 &\quad+\sqrt{\frac{17000\log(48/\delta)}k
 \left(
 \sqrt{\frac{L^*(d+\log(48/\delta))}{k}}
 +\frac{d+\log(48/\delta)}{k}
 \right)}
 \\
 &\le(1+\sqrt{17000})\left(
 \sqrt{\frac{L^*(d+\log(48/\delta))}{k}}
 +\frac{d+\log(48/\delta)}{k}
 \right)
 \\
 &<132\left(
 \sqrt{\frac{L^*(d+\log(48/\delta))}{k}}
 +\frac{d+\log(48/\delta)}{k}
 \right).
\end{align*}
For the second inequality, we used
$\log(48/\delta)/k\le(d+\log(48/\delta))/k$, which is at most the
expression in parentheses.  Substituting these bounds into the validation
inequality gives
\[
 L(\widehat h)<L^*+3.1\cdot10^7\left(
 \sqrt{\frac{L^*(d+\log(48/\delta))}{k}}
 +\frac{d+\log(48/\delta)}{k}
 \right).
\]
A union bound shows that the two events hold simultaneously with probability
at least $1-5\delta/16$, and hence at least $1-\delta$.  Since $k\ge n/6$,
the last display implies \eqref{eq:pac-main}.  When $n<6$, return the
classifier identically equal to $-1$.  The error term in
\eqref{eq:pac-main} exceeds one, so the bound is immediate.

Finally, suppose that $0<\delta\le1/2$.  Since
$d+\log(1/\delta)\ge1+\log 2$ and
$\log 48<\frac52(1+\log 2)$,
we have
$
 d+\log(48/\delta)
 \le\frac72\bigl(d+\log(1/\delta)\bigr).
$
Thus \eqref{eq:pac-main} also implies the bound stated in the abstract.
\end{proof}

\section*{Acknowledgements}

The work of Markus Engelund Mathiasen is supported by the European Union (ERC, TUCLA, 101125203).  Views
and opinions expressed are those of the authors only and do not necessarily
reflect those of the European Union or the European Research Council.  Neither
the European Union nor the granting authority can be held responsible for
them.

{\scriptsize
\let\originalthebibliography\thebibliography
\renewcommand{\thebibliography}[1]{%
 \originalthebibliography{#1}%
 \setlength{\itemsep}{0pt}}
\bibliographystyle{plainurl}
\bibliography{references}}

@inproceedings{AdenAliEtAl2023Optimal,
  author    = {Ishaq Aden-Ali and Yeshwanth Cherapanamjeri and Abhishek Shetty and Nikita Zhivotovskiy},
  title     = {Optimal {PAC} Bounds without Uniform Convergence},
  booktitle = {2023 IEEE 64th Annual Symposium on Foundations of Computer Science (FOCS)},
  year      = {2023},
  pages     = {1203--1223},
  publisher = {IEEE Computer Society},
  address   = {Los Alamitos, CA, USA},
  doi       = {10.1109/FOCS57990.2023.00071}
}

@inproceedings{AdenAliEtAl2023Negative,
  author    = {Ishaq Aden-Ali and Yeshwanth Cherapanamjeri and Abhishek Shetty and Nikita Zhivotovskiy},
  title     = {The {One-Inclusion Graph} Algorithm is not Always Optimal},
  booktitle = {Proceedings of Thirty Sixth Conference on Learning Theory},
  year      = {2023},
  editor    = {Gergely Neu and Lorenzo Rosasco},
  volume    = {195},
  series    = {Proceedings of Machine Learning Research},
  pages     = {72--88},
  publisher = {PMLR},
  url       = {https://proceedings.mlr.press/v195/aden-ali23a.html}
}

@inproceedings{AlonEtAl2021,
  author    = {Noga Alon and Omri Ben-Eliezer and Yuval Dagan and Shay Moran and Moni Naor and Eylon Yogev},
  title     = {Adversarial Laws of Large Numbers and Optimal Regret in Online Classification},
  booktitle = {Proceedings of the 53rd Annual ACM SIGACT Symposium on Theory of Computing},
  series    = {STOC 2021},
  year      = {2021},
  pages     = {447--455},
  publisher = {Association for Computing Machinery},
  address   = {New York, NY, USA},
  isbn      = {9781450380539},
  doi       = {10.1145/3406325.3451041}
}

@inproceedings{AsilisEtAl2024,
  author    = {Julian Asilis and Siddartha Devic and Shaddin Dughmi and Vatsal Sharan and Shang-Hua Teng},
  title     = {Regularization and Optimal Multiclass Learning},
  booktitle = {Proceedings of Thirty Seventh Conference on Learning Theory},
  year      = {2024},
  editor    = {Shipra Agrawal and Aaron Roth},
  volume    = {247},
  series    = {Proceedings of Machine Learning Research},
  pages     = {260--310},
  publisher = {PMLR},
  url       = {https://proceedings.mlr.press/v247/asilis24a.html}
}

@inproceedings{AsilisHogsgaardVelegkas2025,
  author        = {Julian Asilis and Mikael M{\o}ller H{\o}gsgaard and Grigoris Velegkas},
  title         = {On Agnostic {PAC} Learning in the Small Error Regime},
  booktitle     = {Advances in Neural Information Processing Systems},
  year          = {2025},
  volume        = {38},
  pages         = {123346--123388},
  publisher     = {Curran Associates, Inc.},
  url           = {https://papers.nips.cc/paper_files/paper/2025/hash/b2a2bd5d5051ff6af52e1ef60aefd255-Abstract-Conference.html}
}

@article{Audibert2009,
  author  = {Jean-Yves Audibert},
  title   = {Fast Learning Rates in Statistical Inference through Aggregation},
  journal = {The Annals of Statistics},
  year    = {2009},
  volume  = {37},
  number  = {4},
  pages   = {1591--1646},
  doi     = {10.1214/08-AOS623}
}

@article{BartlettBousquetMendelson2005,
  author  = {Peter L. Bartlett and Olivier Bousquet and Shahar Mendelson},
  title   = {Local {Rademacher} Complexities},
  journal = {The Annals of Statistics},
  year    = {2005},
  volume  = {33},
  number  = {4},
  pages   = {1497--1537},
  doi     = {10.1214/009053605000000282}
}

@article{BlumerEhrenfeuchtHausslerWarmuth1989,
  author  = {Anselm Blumer and Andrzej Ehrenfeucht and David Haussler and Manfred K. Warmuth},
  title   = {Learnability and the {Vapnik--Chervonenkis} Dimension},
  journal = {Journal of the ACM},
  year    = {1989},
  volume  = {36},
  number  = {4},
  pages   = {929--965},
  doi     = {10.1145/76359.76371}
}

@article{BoucheronBousquetLugosi2005,
  author  = {St{\'e}phane Boucheron and Olivier Bousquet and G{\'a}bor Lugosi},
  title   = {Theory of Classification: A Survey of Some Recent Advances},
  journal = {ESAIM: Probability and Statistics},
  year    = {2005},
  volume  = {9},
  pages   = {323--375},
  doi     = {10.1051/ps:2005018}
}

@book{DevroyeGyorfiLugosi1996,
  author    = {Luc Devroye and L{\'a}szl{\'o} Gy{\"o}rfi and G{\'a}bor Lugosi},
  title     = {A Probabilistic Theory of Pattern Recognition},
  series    = {Stochastic Modelling and Applied Probability},
  volume    = {31},
  publisher = {Springer},
  address   = {New York, NY},
  year      = {1996},
  doi       = {10.1007/978-1-4612-0711-5}
}

@article{DevroyeLugosi1995,
  author  = {Luc Devroye and G{\'a}bor Lugosi},
  title   = {Lower Bounds in Pattern Recognition and Learning},
  journal = {Pattern Recognition},
  year    = {1995},
  volume  = {28},
  number  = {7},
  pages   = {1011--1018},
  doi     = {10.1016/0031-3203(94)00141-8}
}

@book{DevroyeLugosi2001,
  author    = {Luc Devroye and G{\'a}bor Lugosi},
  title     = {Combinatorial Methods in Density Estimation},
  series    = {Springer Series in Statistics},
  publisher = {Springer},
  address   = {New York, NY},
  year      = {2001},
  doi       = {10.1007/978-1-4613-0125-7}
}

@book{Diestel2017,
  author    = {Reinhard Diestel},
  title     = {Graph Theory},
  edition   = {5th},
  series    = {Graduate Texts in Mathematics},
  volume    = {173},
  publisher = {Springer},
  address   = {Berlin, Heidelberg},
  year      = {2017},
  doi       = {10.1007/978-3-662-53622-3}
}

@inproceedings{DughmiKalayciYork2025,
  author    = {Shaddin Dughmi and Yusuf Hakan Kalayci and Grayson York},
  title     = {Is Transductive Learning Equivalent to {PAC} Learning?},
  booktitle = {Proceedings of The 36th International Conference on Algorithmic Learning Theory},
  year      = {2025},
  editor    = {Gautam Kamath and Po-Ling Loh},
  volume    = {272},
  series    = {Proceedings of Machine Learning Research},
  pages     = {418--443},
  publisher = {PMLR},
  url       = {https://proceedings.mlr.press/v272/dughmi25a.html}
}

@article{Dudley1978,
  author  = {Richard M. Dudley},
  title   = {Central Limit Theorems for Empirical Measures},
  journal = {The Annals of Probability},
  year    = {1978},
  volume  = {6},
  number  = {6},
  pages   = {899--929},
  doi     = {10.1214/aop/1176995384}
}

@inproceedings{FilmusEtAl2023,
  author    = {Yuval Filmus and Steve Hanneke and Idan Mehalel and Shay Moran},
  title     = {Optimal Prediction Using Expert Advice and Randomized {Littlestone} Dimension},
  booktitle = {Proceedings of Thirty Sixth Conference on Learning Theory},
  year      = {2023},
  editor    = {Gergely Neu and Lorenzo Rosasco},
  volume    = {195},
  series    = {Proceedings of Machine Learning Research},
  pages     = {773--836},
  publisher = {PMLR},
  url       = {https://proceedings.mlr.press/v195/filmus23a.html}
}

@article{Freedman1975,
  author  = {David A. Freedman},
  title   = {On Tail Probabilities for Martingales},
  journal = {The Annals of Probability},
  year    = {1975},
  volume  = {3},
  number  = {1},
  pages   = {100--118},
  doi     = {10.1214/aop/1176996452}
}

@article{Hanneke2016Optimal,
  author  = {Steve Hanneke},
  title   = {The Optimal Sample Complexity of {PAC} Learning},
  journal = {Journal of Machine Learning Research},
  year    = {2016},
  volume  = {17},
  number  = {38},
  pages   = {1--15},
  url     = {https://www.jmlr.org/papers/v17/15-389.html}
}

@inproceedings{HannekeLarsenZhivotovskiy2024,
  author    = {Steve Hanneke and Kasper Green Larsen and Nikita Zhivotovskiy},
  title     = {Revisiting Agnostic {PAC} Learning},
  booktitle = {2024 IEEE 65th Annual Symposium on Foundations of Computer Science (FOCS)},
  year      = {2024},
  pages     = {1968--1982},
  publisher = {IEEE Computer Society},
  address   = {Los Alamitos, CA, USA},
  doi       = {10.1109/FOCS61266.2024.00118}
}

@article{HLW1994,
  author  = {David Haussler and Nick Littlestone and Manfred K. Warmuth},
  title   = {Predicting $\{0,1\}$-Functions on Randomly Drawn Points},
  journal = {Information and Computation},
  year    = {1994},
  volume  = {115},
  number  = {2},
  pages   = {248--292},
  doi     = {10.1006/inco.1994.1097}
}

@inproceedings{Larsen2023Bagging,
  author    = {Kasper Green Larsen},
  title     = {Bagging is an Optimal {PAC} Learner},
  booktitle = {Proceedings of Thirty Sixth Conference on Learning Theory},
  year      = {2023},
  editor    = {Gergely Neu and Lorenzo Rosasco},
  volume    = {195},
  series    = {Proceedings of Machine Learning Research},
  pages     = {450--468},
  publisher = {PMLR},
  url       = {https://proceedings.mlr.press/v195/larsen23a.html}
}

@article{Long1999,
  author  = {Philip M. Long},
  title   = {The Complexity of Learning According to Two Models of a Drifting Environment},
  journal = {Machine Learning},
  year    = {1999},
  volume  = {37},
  number  = {3},
  pages   = {337--354},
  doi     = {10.1023/A:1007666507971},
  note    = {Expanded version of the paper in the Proceedings of the Eleventh Annual Conference on Computational Learning Theory (COLT 1998), pp. 116--125, doi:10.1145/279943.279968}
}

@book{ODonnell2014,
  author    = {Ryan O'Donnell},
  title     = {Analysis of Boolean Functions},
  publisher = {Cambridge University Press},
  address   = {New York, NY},
  year      = {2014},
  doi       = {10.1017/CBO9781139814782}
}

@misc{RawalZhivotovskiy2026,
  author        = {Divit Rawal and Nikita Zhivotovskiy},
  title         = {{Majority-of-Three} is Optimal},
  year          = {2026},
  eprint        = {2606.13614},
  archivePrefix = {arXiv},
  primaryClass  = {stat.ML},
  note          = {Version 1, submitted 11 June 2026}
}

@inproceedings{Simon2015,
  author    = {Hans U. Simon},
  title     = {An Almost Optimal {PAC} Algorithm},
  booktitle = {Proceedings of the 28th Conference on Learning Theory},
  year      = {2015},
  editor    = {Peter Gr{\"u}nwald and Elad Hazan and Satyen Kale},
  volume    = {40},
  series    = {Proceedings of Machine Learning Research},
  pages     = {1552--1563},
  publisher = {PMLR},
  url       = {https://proceedings.mlr.press/v40/Simon15a.html}
}

@article{Valiant1984,
  author  = {Leslie G. Valiant},
  title   = {A Theory of the Learnable},
  journal = {Communications of the ACM},
  year    = {1984},
  volume  = {27},
  number  = {11},
  pages   = {1134--1142},
  doi     = {10.1145/1968.1972}
}

@article{VapnikChervonenkis1971,
  author  = {Vladimir N. Vapnik and Alexey Ya. Chervonenkis},
  title   = {On the Uniform Convergence of Relative Frequencies of Events to Their Probabilities},
  journal = {Theory of Probability \& Its Applications},
  year    = {1971},
  volume  = {16},
  number  = {2},
  pages   = {264--280},
  doi     = {10.1137/1116025}
}

\appendix

\section{Development of the proof and AI disclosure}

This project began in spring 2026 while the first author was visiting UC
Berkeley.  From the outset, the authors sought to combine two
ideas.
The Boolean cube orientation of Long \cite{Long1999} gives an agnostic leave
one out bound of the correct order, but with a coefficient greater than one
on the empirical optimum.  Dughmi, Kalayci, and York
\cite{DughmiKalayciYork2025} relate the discounted edge density of the agnostic
one inclusion graph to empirical Rademacher complexity, using the full
coordinate set.

An analogy with agnostic online classification also proved useful.  Alon,
Ben-Eliezer, Dagan, Moran, Naor, and Yogev \cite{AlonEtAl2021} obtain the
optimal horizon dependent regret bound in terms of Littlestone dimension.
Filmus, Hanneke, Mehalel, and Moran \cite{FilmusEtAl2023} localize this
dependence to the $k$ mistakes of the best comparator and obtain the optimal
randomized expected mistake bound
$k+\Theta(\sqrt{k\operatorname{Ldim}(\Hc)}+\operatorname{Ldim}(\Hc))$.
This suggested the localized bound in
\Cref{lem:edge-rad}.

The main obstacle was whether an isoperimetric result such as
\Cref{lem:edge-rad} could hold at all.  In May 2026, the authors used OpenAI
GPT-5.5 Pro to study simple VC classes such as 
thresholds and intervals.  In early July, automated runs with the same model
suggested an argument for classes of VC dimension one based on their tree
structure.  The authors then examined axis parallel
rectangles.  In early August, experiments with GPT-5.6 Sol in Ultra mode
helped resolve the rectangle case.  These examples led the authors to the general random
restriction argument in \Cref{lem:edge-rad}. 

The PAC conversion combines suffix averaging
from Aden-Ali, Cherapanamjeri, Shetty, and Zhivotovskiy
\cite{AdenAliEtAl2023Optimal} with a comparator dependent martingale bound.

The single prompt below compresses these
experiments.  It states \Cref{lem:edge-rad} without a proof hint and gives
more detailed hints for the remaining steps.  On August 6, 2026, the authors
tested it in 16 separate runs
with GPT-5.6 Sol in Pro mode.  In 11 of these runs, the model
produced an essentially correct proof outline for
\Cref{lem:edge-rad} and completed the remaining steps.  The runs took
123 minutes on average.
\vskip+13pt
\begin{promptbox}
Let \(\mathcal H\subseteq\{-1,+1\}^{\mathcal X}\) have VC dimension
\(d\ge1\). Write \(L(h)=\mathbb P(h(X)\ne Y)\) and
\(L^*=\min_{h\in\mathcal H}L(h)\), assuming measurability and attainment.
Given an i.i.d. sample of size \(n\), construct a deterministic, generally
improper classifier \(\widehat h:\mathcal X\to\{-1,+1\}\), without using
\(L^*\) or \(\delta\), such that for every distribution, \(n\ge1\), and
\(\delta\in(0,1)\), with probability at least \(1-\delta\),
\[
L(\widehat h)\le L^*+C\left(
\sqrt{\frac{L^*(d+\log(1/\delta))}{n}}
+\frac{d+\log(1/\delta)}{n}
\right),
\]
where \(C>0\) is a universal constant. Give a complete proof.

Hints. 1) First construct a coefficient-one fractional orientation of the
full Boolean cube. Let
\[
V\subseteq\{-1,+1\}^m
\]
be a finite class with \(\operatorname{VC}(V)\le d\). For
\(y\in\{-1,+1\}^m\), define
\[
\rho_V(y)=\min_{v\in V}d_H(y,v).
\]
For \(D\subseteq[m]\), define the projected Rademacher width
\[
\operatorname{Rad}_D(V)
=\mathbb E_\varepsilon
\sup_{v\in V}
\sum_{i\in D}\varepsilon_i v_i.
\]
Prove the following stronger statement. There exist weights
\[
w_{y,i}\in[0,1],
\qquad y\in\{-1,+1\}^m,\ i\in[m],
\]
such that
\[
w_{y,i}+w_{y^{\oplus i},i}=1
\]
for every vertex \(y\) and coordinate \(i\), and such that for every
\(y\in\{-1,+1\}^m\) and every coordinate set \(D\subseteq[m]\),
\[
\sum_{i\in D}w_{y,i}
\le
\rho_V(y)+C_1\operatorname{Rad}_D(V).
\]
2) For \(U\subseteq\{-1,+1\}^m\), with \(D_U(y)=\{i:y^{\oplus i}\in U\}\)
and internal cube edges \(E_U\), deduce
\(|E_U|\le\sum_{y\in U}(\rho_V(y)+C_1\operatorname{Rad}_{D_U(y)}(V))\).
To localize the Rademacher term, use that \(\rho_V\) is \(1\)-Lipschitz
along cube edges, bound the VC dimension of the Hamming neighborhood
\(\{y:\rho_V(y)\le t\}\) by \(C(d+t)\), and combine the edge-density bound
for VC classes with Cauchy--Schwarz over the level sets of \(\rho_V\).
3) Combine this with \(\operatorname{Rad}_D(V)\le C\sqrt{d|D|}\) and the
Hall orientation criterion to obtain
\(\operatorname{out}(y)\le\rho_V(y)+C(\sqrt{d\rho_V(y)}+d)\); apply this
to the trace of \(\mathcal H\) to obtain a coefficient-one leave-one-out
bound.
4) Symmetrize the rule and average over the half-sample suffix
\(t=k,\ldots,2k-1\); use a backward martingale to pass from the suffix
leave-one-out bounds to the realized held-out errors, and a forward
martingale to pass to the population risk; then derandomize by relative
validation over the nested threshold classifiers, without a logarithmic
factor. Fix all tie rules.
\end{promptbox}
\end{document}